\documentclass[11pt,a4paper]{article}
\usepackage[hyperref]{emnlp2020}
\usepackage{times}
\usepackage{latexsym}

\usepackage{amsmath}
\usepackage{amsfonts}
\usepackage{amsthm}
\usepackage{xspace}
\usepackage{nicefrac}
\usepackage{xfrac}
\usepackage{enumerate}
\usepackage{scalerel}
\usepackage{mathtools}
\usepackage{csquotes}
\usepackage{adjustbox}
\usepackage{wrapfig}
\usepackage{booktabs}
\usepackage{floatrow}
\usepackage{relsize,etoolbox}
\AtBeginEnvironment{quote}{\small}

\theoremstyle{plain}
\newtheorem{theorem}{Theorem}[section]
\newtheorem{hypothesis}{Hypothesis}[section]

\theoremstyle{definition}

\usepackage{cleveref}
\crefname{section}{\S}{\S\S}
\Crefname{section}{\S}{\S\S}
\crefname{table}{Tab.}{}
\crefname{figure}{Fig.}{}
\crefname{algorithm}{Algorithm}{}
\crefname{equation}{Eq.}{Eq.}
\crefname{appendix}{App.}{}
\crefname{theorem}{Theorem}{}
\crefname{prop}{Proposition}{}
\crefname{cor}{Corollary}{}
\crefname{observation}{Observation}{}
\crefname{assumption}{Assumption}{}
\crefname{hypothesis}{Hyp.}{Hypotheses}
\crefformat{section}{\S#2#1#3}

\usepackage{todonotes}
\newcommand*\iftodonotes{\if@todonotes@disabled\expandafter\@secondoftwo\else\expandafter\@firstoftwo\fi}  
\makeatother

\newcommand{\bleu}{\textsc{bleu}\xspace}
\newcommand{\defn}[1]{\textbf{#1}}
\newcommand{\xx}{\mathbf{x}}
\newcommand{\vv}{\mathbf{v}}
\newcommand{\yy}{\mathbf{y}}

\newcommand{\calY}{\mathcal{Y}}

\newcommand{\calB}{\mathcal{B}}
\newcommand{\calR}{\mathcal{R}}
\newcommand{\nmax}{n_\mathrm{max}}
\newcommand{\defeq}{\vcentcolon=}

\newcommand{\vtheta}{{\boldsymbol \theta}}
\newcommand{\ptheta}{p_{\scaleto{\vtheta}{4pt}}}

\newcommand{\vocab}{\mathcal{V}}
\newcommand{\vocabeos}{\bar{\mathcal{V}}}
\newcommand{\eos}{\textsc{eos}\xspace}
\newcommand{\bos}{\textsc{bos}\xspace}

\DeclareMathSymbol{\shortminus}{\mathbin}{AMSa}{"39}

\DeclareMathOperator*{\argmax}{argmax}

\usepackage{emoji}

\usepackage{microtype}

\aclfinalcopy 
\def\aclpaperid{1377} 

\usepackage[utf8]{inputenc}
\newcommand{\ucambridge}{\emoji[twitter]{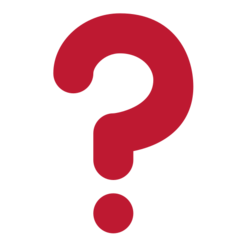}}
\newcommand{\ethz}{\emoji[twitter]{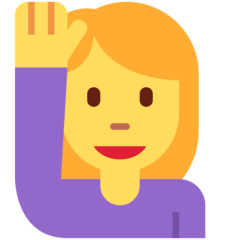}}
\newcommand{\jhu}{\emoji[twitter]{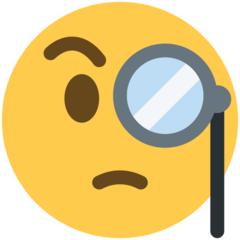}}

\title{If Beam Search is the Answer, What was the Question?}

\author{Clara Meister\raise1.0ex\hbox{\normalfont\ethz}~\;~Tim Vieira\raise1.0ex\hbox{\normalfont\jhu}~\;~Ryan Cotterell\raise1.0ex\hbox{\normalfont\ucambridge\!,\ethz} \\
  \raise1.0ex\hbox{\normalfont\ethz}ETH Z\"{u}rich~\;~\raise1.0ex\hbox{\normalfont\jhu}Johns Hopkins University~\;~\raise1.0ex\hbox{\normalfont\ucambridge}University of Cambridge \\
  \texttt{clara.meister@inf.ethz.ch}~\;~\texttt{tim.vieira@gmail.com} \\ \texttt{ryan.cotterell@inf.ethz.ch}
  }
  
\date{}

\begin{document}

\maketitle
\begin{abstract}
Quite surprisingly, exact maximum a posteriori (MAP) decoding of neural language generators frequently leads to low-quality results \cite{stahlberg_nmt_2019}. Rather, most state-of-the-art results on language generation tasks are attained using beam search despite its overwhelmingly high search error rate. This implies that the MAP objective alone does not express the properties we desire in text, which merits the question: if beam search is the answer, what was the question? We frame beam search as the exact solution to a different decoding objective in order to gain insights into \emph{why} high probability under a model alone may not indicate adequacy. We find that beam search enforces \emph{uniform information density} in text, a property motivated by cognitive science.  We suggest a set of decoding objectives that explicitly enforce this property and find that exact decoding with these objectives alleviates the problems encountered when decoding poorly calibrated language generation models. Additionally, we analyze the text produced using various decoding strategies and see that, in our neural machine translation experiments, the extent to which this property is adhered to strongly correlates with \bleu.
Our code is publicly available at \url{https://github.com/rycolab/uid-decoding}.
\end{abstract}

\section{Introduction}
As a simple search heuristic, beam search has been used to decode models developed by the NLP community for decades.  Indeed, it is noteworthy that beam search is one of the few NLP
algorithms that has stood the test of time: It has remained a cornerstone of NLP systems since the 1970s \cite{reddy-1977}. As such, it became the natural choice for decoding neural probabilistic text generators---whose design makes evaluating the full search space impossible \cite{kalchbrenner-blunsom-2013-recurrent, sutskever_seq2seq, neural_conv, yin-etal-2016-neural-generative}. While there is no formal guarantee that beam search will return---or even approximate---the highest-scoring candidate under a model, it has repeatedly proven its merit in practice \cite{AAAI1714571, edunov-etal-2018-understanding, XLNET} and, thus, has largely been tolerated---even embraced---as NLP's go-to search heuristic. 
\begin{figure}
\centering
    \includegraphics[width=\linewidth]{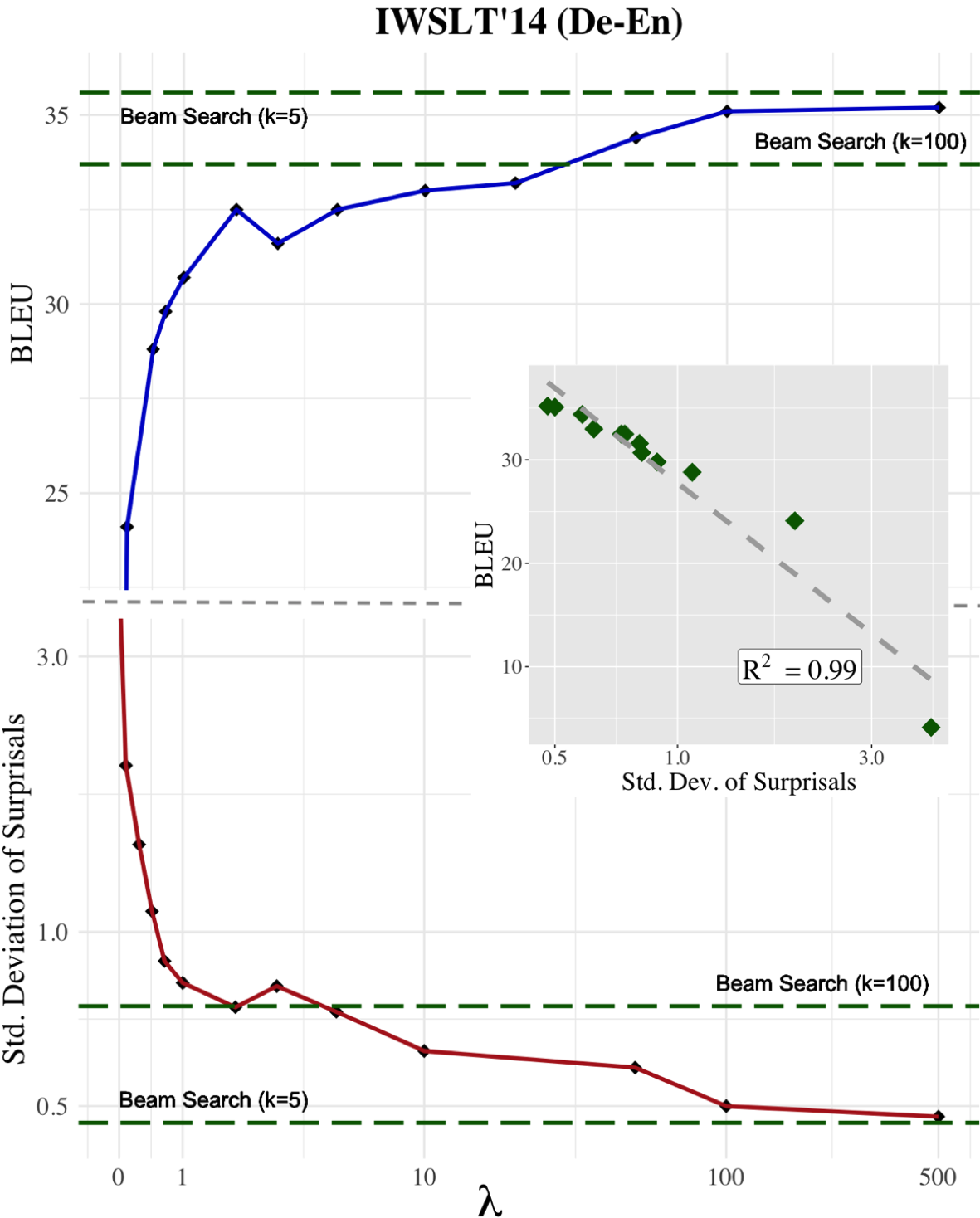}
  \caption{Average std.\@ deviation $\sigma$ of surprisals (per sentence) and corpus \bleu for translations generated using exact search over the MAP objective with a greedy regularizer (\cref{eq:greedy}) with varying degrees of $\lambda$. References for beam search ($k=5$ and $k=100$) are included. Sub-graph shows the explicit relationship between \bleu and $\sigma$. $\lambda$ and $\sigma$ axes are log-scaled.}
  \setlength{\belowcaptionskip}{-30pt}
    \label{fig:bleu_surprisal}
\end{figure}
\indent However, in the context of neural machine translation (NMT), a shocking empirical finding has emerged:  Using beam search to decode sentences from neural text generators almost invariably leads to better text than using exact search (or beam search with a very large beam size).  In fact, \newcite{stahlberg_nmt_2019} report that exact search returns the empty string in $>50\%$ of cases,\footnote{This rate tends to decrease for larger models, although it is often still a considerable percentage.} showing that the success of beam search does not stem from its ability to approximate exact decoding in practice, but rather due to a hidden inductive bias embedded in the algorithm. This inductive bias appears to be \emph{paramount} for generating desirable text from neural probabilistic text generators.
While several works explore this phenomenon \cite{murray-chiang-2018-correcting, yang-etal-2018-breaking, stahlberg_nmt_2019,pmlr-v97-cohen19a}, no one has yet hypothesized what beam search's hidden inductive bias may be. Our work fills this gap.

We analyze the \defn{beam search blessing} by reverse engineering an objective that beam search returns the exact solution for. Specifically, we introduce a regularizer for the the standard (MAP) decoding objective for text generation models such that the exact solution to this regularized objective is equivalent to the solution found by beam search under the unmodified objective. Qualitative inspection reveals that our ``beam search regularizer'' has a clear connection to a theory in cognitive science---the \textbf{uniform information density} hypothesis \cite[UID;][]{NIPS2006_3129}. The UID hypothesis states that---subject to the constraints of the grammar---humans prefer sentences that distribute information (in the sense of information theory) equally across the linguistic signal, e.g., a sentence. In other words, human-produced text, regardless of language, tends to have evenly distributed surprisal, formally defined in information theory as negative log-probability. This connection suggests beam search has an interpretation as exact decoding, but with a UID-promoting regularizer that encourages evenly distributed surprisal in generated text. This insight naturally leads to the development of several new regularizers that likewise enforce the UID property.

Empirically, we experiment with our novel regularizers in the decoding of NMT models. We first observe a close relationship between the standard deviation of surprisals---an operationalization of UID---and \bleu, which suggests that high-quality text does indeed exhibit the UID property. Additionally, we find that even with exact search, our regularized objective leads to performance similar to beam search on standard NMT benchmarks. Both of these observations are reflected in \cref{fig:bleu_surprisal}.
Lastly, we see that our regularizers alleviate the text-quality degradation typically seen when decoding with larger beam sizes. We take all the above as evidence that our proposed explanation of beam search's inductive bias indeed elucidates \emph{why} the algorithm performs so well as a search heuristic for language generation tasks.

\section{Neural Probabilistic Text Generation}
Probabilistic text generators
define a probability distribution $\ptheta(\yy \mid \xx)$
over an output space of hypotheses $\calY$ (to be defined in \cref{eq:candidates}) conditioned on an input $\xx$.\footnote{The input could be another sentence, a semantic structure or an image, to name a few examples.} Modern generators are typically parameterized by a deep neural network---possibly recurrent---with a set of learned weights $\vtheta$. 
In the case of text generation, the full set of possible hypotheses grows exponentially with the vocabulary size $|\vocab|$. We consider the set of complete hypotheses, i.e., valid outputs, as
\begin{equation}\label{eq:candidates}
    \calY \defeq \{ \bos \circ \vv \circ \eos \mid \vv \in \vocab^* \}
\end{equation} 
\noindent where $\circ$ is string concatenation and $\vocab^*$ is the Kleene closure of $\vocab$. In words, valid hypotheses are text, e.g., sentences or phrases, padded with distinguished tokens, \bos and \eos.
In this work, we consider models that are locally normalized, i.e., the model $\ptheta$ is defined as the product of probability distributions:
\begin{equation}
    p_\vtheta(\yy \mid \xx) = \prod_{t=1}^{|\yy|}p_\vtheta(y_t \mid \xx, \yy_{<t})
\end{equation}
\noindent where each $\ptheta(\cdot \mid \xx, \yy_{<t})$ is a distribution with support
over $\vocabeos \defeq \vocab \cup \{\eos\}$ and $\yy_{<1} = y_0  \defeq \bos$.\looseness=-1

The decoding objective for text generation aims to find the most-probable hypothesis among all candidate hypotheses, i.e.
we aim to solve the following optimization problem:
\begin{equation}
    \yy^\star = \argmax_{\yy \in \calY} \log \ptheta(\yy \mid \xx)
    \label{eq:MAP}
\end{equation}
This is commonly known as maximum a posteriori (MAP)
decoding since $\ptheta$ is a probability model. While there exists a wealth of literature on decoding algorithms for statistical text generation models, e.g., phrase-based machine translation models, many of these methods cannot reasonably be used with neural models. Specifically, due to the non-Markovian structure of most neural text generators, dynamic-programming algorithms for searching over the exponentially large space are not efficient in this setting. Indeed, there are formal results that solving \cref{eq:MAP} with a recurrent neural network is NP-hard \cite{chen-etal-2018-recurrent}. Therefore decoding is performed almost exclusively with heuristic methods, such as beam search.

\subsection{Beam Search}
Beam search is a form of pruned breadth-first search where the breadth is limited to $k\in \mathbb{Z}_+$ (i.e., a maximum of $k$ hypotheses) are expanded at each time step. We express beam search as the following recursion:
\begin{align}
    Y_0 &= \{\bos\} \\
    Y_t &= \argmax_{\substack{Y' \subseteq \calB_t, \\ |Y'| = k}}\ \log \ptheta(Y' \mid \xx) 
    \label{eq:beam}
\end{align}
where we define the candidate set at $t>0$
\begin{equation}
    \calB_t = \Big\{\yy_{t\shortminus 1} \circ y \mid y \in \vocabeos \textbf{ and } \yy_{t \shortminus 1} \in Y_{t\shortminus 1} \Big\}
\end{equation} 
\noindent For notational convenience, we define $\eos \circ \eos = \eos$. The above algorithm terminates after a fixed number of iterations\footnote{If all hypotheses in $Y_{t}$ end in \eos for some $t < \nmax$, then we may terminate beam search early as it is then gauranteed that $Y_{t} = Y_{\nmax}$. We do not consider further early-stopping methods for beam search \cite{huang-etal-2017-finish,yang-etal-2018-breaking, meister+al.tacl20} as they generally should not affect the \emph{quality} of the decoded set.} $\nmax$ and the set $Y_{\nmax}$ is returned. 
We overload $\ptheta(\cdot 
\mid \xx)$ to take a set of hypotheses as an argument instead of just a single hypothesis. In this case, $\ptheta(Y 
\mid \xx) \coloneqq \prod_{\yy \in Y} \ptheta(\yy \mid \xx)$.\footnote{There do exist objectives that take into account interactions between hypotheses in a set, e.g., diverse beam search \cite{diverse-beam-search}, but we do not consider those here.}  Using a similar schema, the $\argmax$ may also operate over a different objective, e.g., log-probabilities combined with various rewards or penaties, such as those discussed in \cref{sec:alt-decode}.

Beam search has a long history in sequence transduction. For example, many of the decoding strategies used in statistical machine translation (SMT) systems were variants of beam search \cite{och-etal-1999-improved, koehn-etal-2003-statistical, pharoah}. As language generation systems moved away from phrase-based statistical approaches and towards neural models, beam search remained the de-facto decoding algorithm \cite{sutskever_seq2seq, neural_conv}. However, it has been observed that when used as a decoding algorithm for neural text generation, beam search (for small beams) typically has a large percentage of search errors \cite{stahlberg_nmt_2019}. Counterintuitively, it is widely known that increasing the beam size beyond $5$ can hurt model performance in
terms of downstream evaluation metrics (e.g., \textsc{bleu}, \textsc{rouge}); while a number of prior works have referred to this phenomenon as a curse \cite{koehn-knowles-2017-six, yang-etal-2018-breaking, pmlr-v97-cohen19a}, it should perhaps be seen as a \emph{blessing}. Beam search typically generates well-formed and coherent text from probabilistic models, whose global optimum in many cases is the empty string, when they otherwise might fail to produce text at all. As we demonstrate in \cref{sec:uid}, this text also tends to be \emph{human-like}. We will subsequently explore possible reasons as to why beam search leads to desirable text from models that are otherwise poorly calibrated, i.e., poor representations of the true distribution $p(\yy\mid\xx)$ \cite{guo_calibration}.\looseness=-1

\subsection{Alternative Decoding Objectives}\label{sec:alt-decode}
When the MAP objective (\cref{eq:MAP}) is used for decoding neural text generators, the results are generally not satisfactory.
Among other problems, the generated texts are often short and defaults to high-frequency words \cite{cho-etal-2014-properties, neural_conv, shen-etal-2016-minimum}.
Methods such as length and coverage normalization \cite{jean-etal-2015-montreal, tu-etal-2016-modeling, murray-chiang-2018-correcting}, which augment the MAP objective with an additive term or multiplicative factor, have been adopted to alleviate these issues. 
For example, two such forms of length\footnote{The predominant form of length normalization divides (log) sequence probability by the length of the hypothesis rather than using an additive reward as in \cite{he-2016-length}. We present results from the former in our experiments as we find it empirically leads to better performance.} and coverage normalization use the following modified MAP objective respectively during decoding to produce higher-quality output:
\begin{align}
    \log\, & \ptheta(\yy\! \mid\! \xx) + \lambda |\yy| \label{eq:length-reward}\\
   \log \ptheta(\yy\! \mid\! \xx) + &\lambda \sum_{i=1}^{|\xx|}\log \min \left(1, \sum_{j=1}^{|\yy|}\alpha_{ij}\right) \label{eq:attention-coverage}
\end{align}
\noindent where $\lambda > 0$ is the (tunable) strength of the reward and $\alpha_{ij}$ is the attention weight \cite{bahdanau2014neural} from the $j^\text{th}$ decoding step over the $i^\text{th}$ input.
\Cref{eq:length-reward} directly rewards longer outputs \cite{he-2016-length} while \cref{eq:attention-coverage} aims to reward coverage of input words in a prediction using the attention mechanism of an encoder--decoder model as an oracle \cite{tu-etal-2016-modeling}.
While such methods help obtain state-of-the-art results in neural MT \cite{Wu2016GooglesNM, conv_nmt, ng-etal-2019-facebook}, 
we view them as a patch to the observed problems. The fact that text quality still degrades with increased beam sizes when these rewards are used \cite{koehn-knowles-2017-six, ott2018analyzing} suggests that they do not address the inherent issues with text generation systems. We subsequently hypothesize about the nature of these issues and provide a set of linguistically motivated regularizers---inspired by beam search---that appear to alleviate them.

\section{Deriving Beam Search}\label{sec:derive}
We introduce
a \defn{regularized decoding} framework. The idea
is simple; we seek to solve the \emph{regularized} optimization problem to decode
\begin{equation}
    \yy^\star = \argmax_{\yy \in \calY} \Big(\log \ptheta(\yy \mid \xx) - \lambda \cdot \calR(\yy)\Big)
    \label{eq:regularized-obj}
\end{equation}
for a strategically chosen $\calR(\cdot)$. Clearly,
for certain $\calR(\cdot)$, we recover the decoding
objectives discussed in \cref{sec:alt-decode}. 
The question we ask in this work is the following:
If we want to view beam search as an exact-decoding algorithm,
which $\calR(\cdot)$ should we choose to recover beam search?

We discovered an elegant answer rooted in information theory and cognitive science (the connections are discussed in-depth in \cref{sec:uid}).  We first define the model's time-dependent surprisals, which are an information-theoretic concept that characterizes the amount of new information expressed at time $t$:
\begin{align}
    u_0(\bos) &= 0 \nonumber \\
    u_t(y) &= -\log \ptheta(y \mid  \xx, \yy_{<t}), \,\,\textbf{for }t \ge 1
\end{align}
Note that minimally surprising means maximally probable.
For the special case of greedy decoding, where $k=1$, the following choice of regularizer recovers beam search for sufficiently large $\lambda$:
\begin{align}
    \!\!\!\calR_{\mathrm{greedy}}(\yy)
     = \sum_{t=1}^{|\yy|} \left(u_t(y_t) - \min_{y' \in \vocab} u_t(y')\right)^2
     \label{eq:greedy}
\end{align}
The intuition behind \cref{eq:greedy} is to encourage locally optimal decisions: Every local surprise $u_t$ should be close
to the minimally surprising choice. In the limiting case where locally optimal decisions are not just encouraged, but rather enforced,
we recover greedy search.

Formally, we have the following theorem:
\begin{theorem}\label{sec:greedy-thm}
The argmax of $\log \ptheta(\yy \mid \xx) - \lambda \cdot \calR_{\mathrm{greedy}}(\yy)$ is exactly computed by greedy search in the limiting case as $\lambda \rightarrow \infty$.
\end{theorem}
\begin{proof}
By induction. In \cref{sec:theory}.
\end{proof}

\cref{sec:greedy-thm} establishes that greedy search is the
limiting case of a regularizer that seeks to encourage decisions to have high-probability \emph{locally}.  In contrast, the optimal MAP solution will generally not have this property.  This is because a globally optimal MAP decoder may require a locally suboptimal decision for the sake of being able to make a \defn{compensatory decision} later that leads to global optimality.\footnote{Indeed, we only
have formal guarantees for greedy algorithms when local optimality translates into global optimality \cite[Chapter~4]{algorithm_design}.} 

We now consider the generalization of greedy search ($k=1$) to full beam search ($k \geq 1$). Recall that beam search returns not just a single output, but rather a \emph{set} of outputs. Thus, we must consider
the set-decoding objective
\begin{equation}\label{eq:set-objective}
    Y^\star = \argmax_{\substack{Y \subseteq \calY, \\ |Y| = k}} \Big(\log \ptheta(Y \mid \xx) - \lambda \cdot \calR(Y)\Big)
\end{equation}
where, as before, we have used our overloaded notation $\ptheta(\cdot \mid \xx)$ to score sets of hypotheses.
Similarly to $\calR_{\mathrm{greedy}}$, we formulate a greedy set-regularizer to recover beam search:
\begin{align}
    \calR_{\mathrm{beam}}(Y) &= \\
    \sum_{t=1}^{\nmax} &\left(u_t(Y_t) - \min_{\substack{Y' \subseteq \calB_t, \\ |Y'| = k}} u_t(Y')\! \right)^2 \nonumber
\end{align}
where $Y_t = \{\yy_{1:t} \mid \yy \in Y \}$ corresponds to the set of hypotheses expanded by $t$ steps.\footnote{This includes both incomplete hypotheses of length $t$ and complete hypotheses that have reached \eos at step $\leq t$.}  Note that we additionally overload surprisal to operate on sets, $u_t(Y) = \sum_{y \in Y} u_t(y)$.
We prove an analogous theorem to \cref{sec:greedy-thm} for this regularizer.
\begin{theorem}\label{sec:beam-limit}
The argmax of $\log \ptheta(Y\mid \xx)  - \lambda \cdot \calR(Y)$ is computed by beam search with beam size of $k=|Y|$ as $\lambda \rightarrow \infty$.
\end{theorem}
\begin{proof}
The proof follows from the same argument as \cref{sec:greedy-thm}, albeit with sets instead of an individual hypothesis.
\end{proof}
\noindent Note that in the (predominant) case where we want to return a single candidate sentence as the output rather than an entire set---as would be generated by \cref{eq:set-objective}---we can take the highest-probability sequence in the chosen set $Y^\star$ as our decoded output.
The objective in \cref{eq:set-objective} boils down to a subset selection problem which, given the size of $\calY$, is a computationally prohibitive optimization problem. Nonetheless, we can use it to analyze the properties enforced on generated text by beam search.\looseness=-1
\begin{figure*}
\centering
    \includegraphics[width=\textwidth]{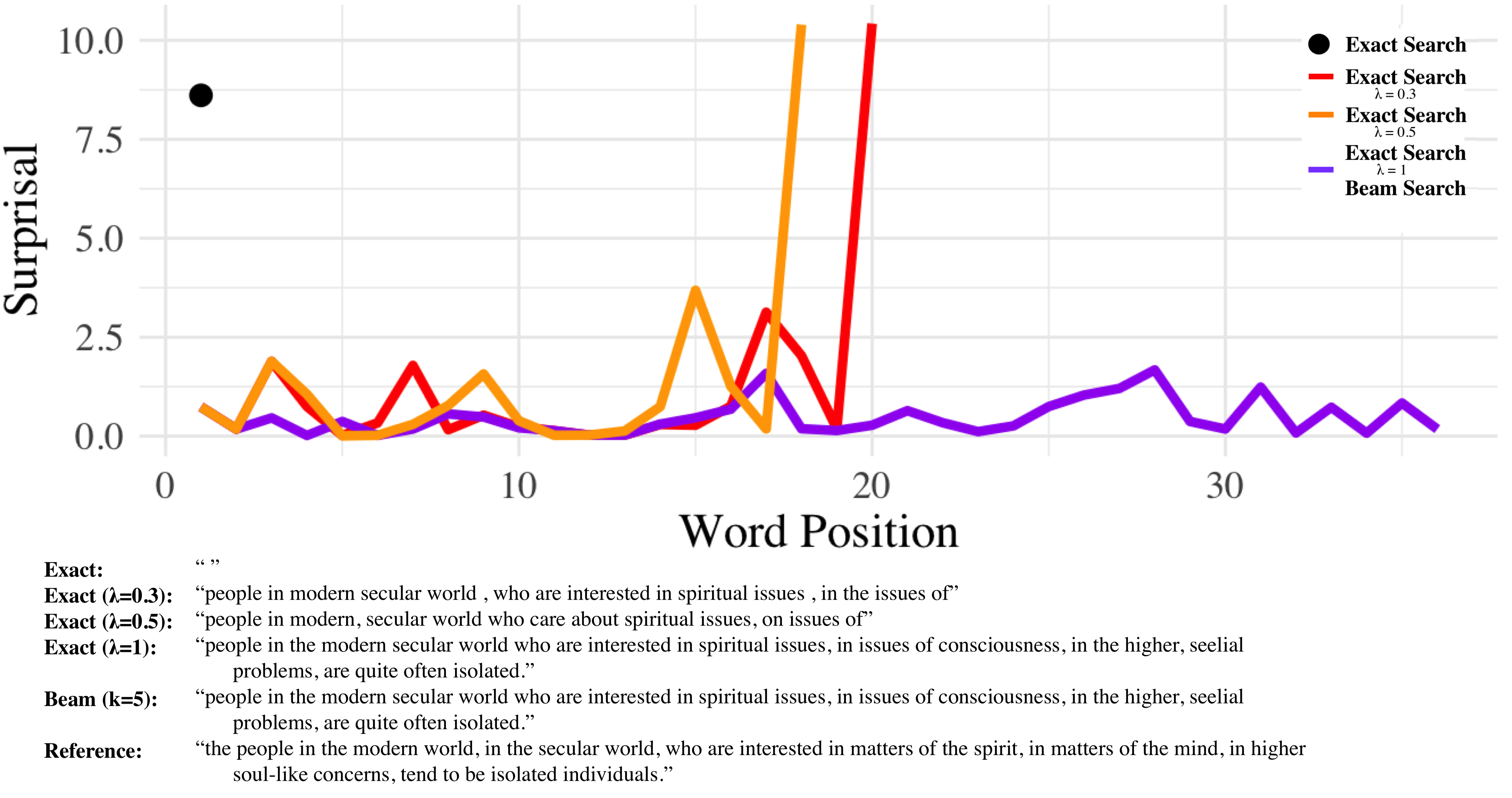}
  \caption{Surprisals (according to $\ptheta$) by time step of sequences generated with various decoding strategies. Values of $\lambda$ indicate the greedy regularizer was used with the corresponding $\lambda$ value. Note that beam search (k=5) and exact search ($\lambda=1.0$) return the same prediction in this example, and thus, are represented by the same line.}
    \label{fig:surprisal}
\end{figure*}

\section{From Beam Search to UID}\label{sec:uid}
The theoretical crux of this paper hinges on a proposed relationship between beam search and the \defn{uniform information density} hypothesis \cite{levy_thesis,NIPS2006_3129}, a concept from cognitive science:\looseness=-1
\begin{hypothesis}\label{hyp:uid}
``Within the bounds defined by grammar, speakers prefer utterances that distribute information uniformly across the signal (information density). Where speakers have a choice between several variants to encode their message, they prefer the variant with more uniform information density (ceteris paribus)'' \cite{jaeger2010redundancy}. 
\end{hypothesis}

At its core, the theory seeks to explain various aspects of human language processing in terms of information theory;
it is often applied to an area of psycholinguistics known as sentence processing where
the UID hypothesis is used to explain experimental data \cite{hale-2001-probabilistic2}. As the UID hypothesis concerns a cognitive process (virtually) independent of the language in use, the theory should hold across languages \cite{jaeger_utility}. 

To see the hypothesis in action, 
consider the classic case of syntactic reduction from \citet{NIPS2006_3129}:

\begin{enumerate}[(1)]
    \item How big is [$_\mathrm{NP}$ the family$_\mathrm{i}$ [$_\mathrm{RC}$ (that) you cook for $_\mathrm{-i}$]]?
\end{enumerate}
In the above example, the sentence does not require the relativizer \textit{that} at the start of the relative clause (denoted by RC); it would also be syntactically correct without it. However, many would agree that the relativizer makes the text qualitatively better. The information-theoretic explanation of this perception is that without the relativizer, the first word of a relative clause conveys two pieces of information simultaneously: the onset of a relative clause and part of its internal contents. Including the relativizer spreads this information across two words, thereby distributing information across the sentence more uniformly and avoiding instances of high surprisal---which, from a psycholinguistic perspective, are displeasing. In short, the relativizer helps to ensure the UID property of the sentence. 

Importantly, the preference suggested by the UID hypothesis is between possible utterances (i.e., outputs) where grammaticality and information content are held constant. Any violation of these assumptions presents confounding factors when measuring, or optimizing, the information density of the generated text.
In our setting, there is reason to believe that grammaticallity and information content are approximately held constant while selecting between hypothesis.
First, the high-probability outputs of neural generation models tend to be grammatical \cite{holtzman2019curious}.
Second, because decoding is conditioned on a specific input $\xx$,
the conditional probability model $p_{\vtheta}(\yy \mid \xx)$ is able to assign high-probability to outputs $\yy$ that are plausible outputs (e.g., translations) of the given $\xx$.  Thus, even though the various $\yy$ are not constrained to be sematically equivalent to one another, they tend to express similar information because they are at least relevant to the same $\xx$.
This is why our regularized optimization problem \cref{eq:regularized-obj} combines
an information-density regularizer with $\log p_{\vtheta}(\yy \mid \xx)$: 
the term $\log p_{\vtheta}(\yy \mid \xx)$ rewards grammaticallity and content relevance, whereas the information-density regularizer encourages the human preferences posited by the UID hypothesis.  The parameter $\lambda$ allows the preferences to be calibrated to perform well on downstream evaluation metrics, such as \textsc{bleu} and \textsc{rouge}.

\subsection{The UID Bias in Beam Search}

It may not be immediately obvious how the UID hypothesis relates to beam search. After all, beam search narrows the scope of the search to only the lowest surprisal candidates at each time step, which does not clearly lead to a uniform distribution of surprisals in the final decoded sequences. The connection is best seen visually. 

\cref{fig:surprisal} shows the time-dependent surprisals $u_t$ under the model of several candidate translations (German to English). Recall that we have $u_t(y) \in [0, \infty)$ and that the standard decoding objective explicitly minimizes the sum of surprisals, i.e., maximizes log-probability. Therefore, the only way the distribution of a solution can become distinctly non-uniform is when there are several high-surprisal decisions in the mix; we observe this in the orange and red curves. 
Intuitively, this corresponds to the notion of compensation discussed earlier: a globally optimal decoding scheme may select a high-surprisal step at some point in order to shorten the length of the path or to take a low-surprisal step later on. 
We observe an extreme example of this behavior above: Selecting the \eos character at the first step leads to a very non-uniform distribution, i.e., the degenerate distribution, which, violates our operationalization of UID described subsequently.
In summary, we see that as $\lambda$ is decreased, the decoded sentences obey the UID property less strictly. 
Indeed, setting $\lambda = 0$, i.e., exact inference of the MAP objective, results in the empty string. 

A number of successful sampling methods ($p$-nucleus sampling \cite{holtzman2019curious} and top-$k$ sampling \cite{fan_hierarchical_2018}) enforce the UID property in generated text by the same logic as above. Both methods eliminate many of the high-surprisal choices at any given decoding step by narrowing the set of tokens that may be chosen. 

\subsection{Cognitive Motivation for Beam Search}
The goal of this work is to expose a possible inductive bias of beam search.
We now exhibit our primary hypothesis
\begin{hypothesis}\label{hyp:beam}
Beam search is a cognitively motivated search heuristic for decoding language generation models. 
The success of beam search on such tasks is, in part, due to the fact that it inherently biases the search procedure towards text that humans prefer.
\end{hypothesis}
The foundation of the argument for this hypothesis follows naturally from the previous sections:
First, we demonstrated in \cref{sec:derive} that
beam search is an exact decoding algorithm for a certain regularized objective---to wit, the one in \cref{eq:regularized-obj}. Qualitatively, we related the behavior of the regularizer to the UID hypothesis from cognitive science. As a final step, we next provide operationalizations of UID---in the form of regularizers within our regularized decoding framework---through which we can empirically test the validity of this hypothesis.

\section{Generalized UID Decoding}

If beam search is trying to optimize for UID, can we beat it at its own game?
This section develops a battery of possible sentence-level UID measures, which can be used as regularizers in our regularized decoding framework and compared experimentally on downstream evaluation metrics.

\paragraph{Variance Regularizer.}
We first consider the variance regularizer from \newcite{jain-etal-2018-uniform}. In essence, UID concerns the distribution of information over the course (i.e., time steps) of a sentence. A natural measure for this is variance of the surprisals. 
\begin{equation}
\calR_{\mathrm{var}}(\yy) = \frac{1}{|\yy|}\sum_{t=1}^{|\yy|} \Big(u_t(y_t) - \mu \Big)^2
\label{eq:var}
\end{equation}
where $\mu = \sfrac{1}{|\yy|}\sum_{t=1}^{|\yy|} u_t(y_t)$. This regularizer, in contrast to \cref{eq:greedy}, is a much more straight-forward encoding of the UID: it directly operationalizes
UID through variance.

\paragraph{Local Consistency.}
Next we consider a local consistency regularizer, also taken from  \newcite{jain-etal-2018-uniform}, that encourages adjacent surprisals to have similar magnitude:
\begin{align}
    \calR_\mathrm{local}(\yy) &= \frac{1}{|\yy|}\sum_{t=1}^{|\yy|} \Big( u_t(y_t) -  u_{t-1}(y_{t-1}) \Big)^2 
    \label{eq:lv}
\end{align}
Again, this is a straightforward encoding of the UID:
if every surprisal is similar to its neighbor, it will be close
to uniform.
Note that both of the above regularizers are defined for all decoding steps $t>0$ since we define $u_0(y_0) = 0$, $y_0 =$ \bos for all valid hypotheses. 

\paragraph{Max Regularizer.}
We propose a UID-inspired regularizer of our own design that exploits
the nature of MAP decoding, for which the overarching goal is to find a solution with low surprisal. In this setting, one strategy is to penalize decisions that move the distribution away from 0, the lowest possible surprisal. This suggests
\begin{equation}
\calR_{\mathrm{max}}(\yy) = \max_{t=1}^{|\yy|} u_t(y_t) 
\label{eq:max}
\end{equation}
would regularize for UID. Such a regularizer would also directly penalize extreme compensation during decoding (discussed in \cref{sec:derive}). It is worth noting that this regularizer has a connection to entropy regularization, which can be seen by looking at the formula for R{\'e}nyi entropy.

\paragraph{Squared Regularizer.}
Finally, we consider a novel squared penalty, that, again, exploits
the goal of MAP decoding. If we wish to keep everything uniform, we can try to push all surprisals close to 0, but this time with a squared penalty:
\begin{equation}
\calR_{\mathrm{square}}(\yy) = \sum_{t=1}^{|\yy|} u_t(y_t)^2
\label{eq:square}
\end{equation}
Experimentally, we expect to see the following: If encouraging decoded text to exhibit UID is helpful---and our logic in constructing regularizers is sound---all the regularizers (\cref{eq:var,eq:lv,eq:max,eq:square}) should lead to roughly the same performance under exact decoding and beam search with large beam widths. Such results would not only validate the connection between UID and high-quality text; comparable performance of optimal beam search\footnote{By optimal beam search, we mean beam search using the beam width that empirically leads to the best results.} and exact search under our regularized objective would provide explicit evidence for our declarative explanation of the inductive bias in beam search.

\begin{figure*}
\centering
    \includegraphics[width=\textwidth]{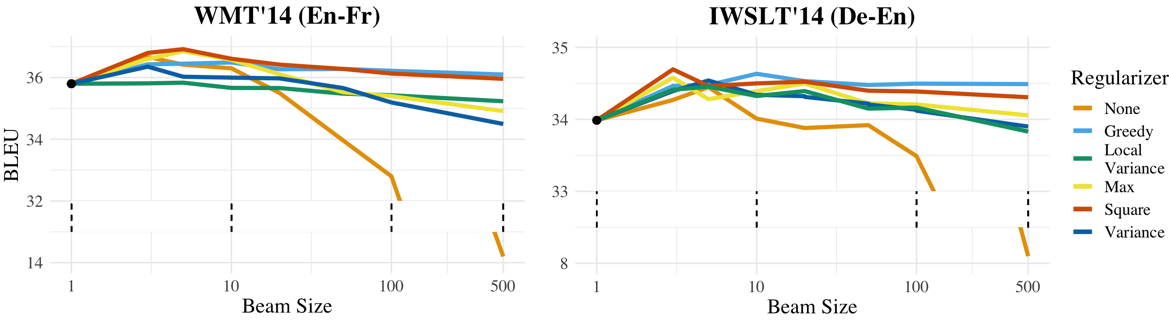}
  \caption{\bleu as a function of beam width for various regularizers. We choose $\lambda$ for each regularizer by best performance on validation sets (see \cref{sec:parameters}). $y$-scales are broken to show minimum \bleu values. $x$-axis is log-scaled.\looseness=-1 }
    \label{fig:beam_bleu}
\end{figure*}

\section{Experiments}
We explore how encouraging uniform information density in text generated by neural probabalistic text generators affects its downstream quality.  To this end, we decode NMT models using the regularized objective (\cref{eq:regularized-obj}) with our UID regularizers. We perform exact decoding for a range of $\lambda$ and observe how text quality (quantified by \bleu \cite{Papineni} using the SacreBLEU \cite{sacrebleu} system) and the distribution of surprisal changes. We additionally evaluate our regularizers under the beam search decoding strategy to see if penalizing violations of 
UID alleviates the text-quality degradation typically seen with increased beam widths.

Experiments are performed using models trained on the IWSLT'14 De-En \cite{IWSLTbib} and WMT'14 En-Fr \cite{WMT14} datasets. For reproducibility, we use the model provided by fairseq \cite{ott2019fairseq} for the WMT'14 task;\footnote{This model uses a transformer architecture \cite{vaswani_att} and was trained as in \citet{ott-etal-2018-scaling}.} we use the data pre-processing scripts and recommended hyperparameter settings provided by fairseq for training a model on the IWSLT'14 De-En dataset. We use the Newstest'14 dataset as the test set for the WMT'14 model. All model and data information can be found on the fairseq NMT repository.
\footnote{ \url{https://github.com/pytorch/fairseq/tree/master/examples/translation}}\looseness=-1

\subsection{Exact Decoding}
To perform exact decoding of neural probabilistic text generators, we build on the decoding framework of \citet{stahlberg-etal-2017-sgnmt}, albeit using Dijkstra's algorithm \cite{dijkstra1959note} instead of depth-first search as we find it decreases decoding time. Note that Dijkstra's algorithm is guaranteed to find the global optimum when path cost is monotonically increasing, which is the case for hypotheses under the scoring scheme used by neural probabilistic text generators (see \citet{meister+al.tacl20} for more detailed discussion). While the variance and local consistency regularizers \cref{eq:var,eq:lv} break this monotonicity property, we can still guarantee optimality by using a stopping criterion similar to the one proposed by \citet{yang-etal-2018-breaking}. Explicitly, we check if the top-scoring complete hypothesis has a greater score than the maximum possible score of any hypothesis in the queue. All scores are bounded due to the maximum-length criterion. Additionally, we lower-bound each search by the score of the empty string to decrease the memory footprint, i.e., we stop considering hypotheses whose scores (or maximum possible score in the case of \cref{eq:var,eq:lv}) drop below that of the empty string at any time step. 

\cref{fig:bleu_surprisal} demonstrates how the addition of the greedy UID regularizer (\cref{eq:greedy} ) to the regularized MAP objective (\cref{eq:regularized-obj}) affects characteristics of the global optimum under the model as we vary $\lambda$. Notably, increasing the strength of the regularizer appears to alleviate the text quality degradation seen with exact search, leading to results that approach the \bleu of those generated using optimal beam search. \cref{fig:bleu_surprisal} also shows a strong inverse relationship between \bleu and average standard deviation (per sentence) of surprisals. We take these observations as empirical validation of \cref{hyp:beam}. 

\subsection{Regularized Beam Search}
We next look at how the regularized decoding objective affects text generated using beam search. As previously noted, text quality generally degrades with increased beam size when using the standard MAP objective; this phenomenon is demonstrated in \cref{fig:beam_bleu}. UID regularization appears to alleviate this problem. Notably, the greedy and squared regularizer aid performance for larger beam sizes more so than other regularizers, for which we still see a slight drop in performance for larger beam sizes. This drop is negligible compared to the one observed for unregularized beam search---a drop which is also frequently observed for length-normalized decoding \cite{koehn-knowles-2017-six}. While intuitively, variance and local variance are the purest encodings of UID, they perform the poorest of the regularizers. Arguably, this may be due to the fact that they do not simultaneously (as the other regularizers do) penalize for high surprisal.

We additionally decode with a combination of the UID regularizers in tandem.
We collectively tune the $\lambda$ value for each of the regularizers on validation sets. We report performance in \cref{tab:results} and see that results outperform standard and length-normalized, i.e. score divided by sequence length, beam search with noticeable improvements for larger beams. Search details and parameter settings may be found in \cref{sec:parameters}. Notably, combining multiple UID regularizers does not lead to as great an increase in performance as one might expect, which hints that a single method for enforcing UID is sufficient for promoting quality in generated text.

\begin{table}
  \centering
  \adjustbox{width=\linewidth}{
  \begin{tabular}{@{}lllll@{}}
  \toprule
     & $k\!=\!5$ & $k\!=\!10$ & $k\!=\!100$ & $k\!=\!500$\,\,  \\
     
    \hline
     No Regularization &36.42 &36.30 & 32.83& 14.66 \rule{0pt}{3ex}\\
     Squared Regularizer & \bf 36.92 & 36.42& 36.13 & 35.96\\
     Greedy Regularizer & 36.45 & 36.49&  36.22 & 36.15\\
     Combined Regularizers & 36.69 & \bf 36.65 &\bf 36.48 & \bf 36.35 \\
     Length Normalization & 36.02 &35.94 & 35.80&  35.11 \\
    
    \bottomrule
  \end{tabular} }
  \caption{\bleu scores on first 1000 samples of Newstest2014 for predictions generated with various decoding strategies. Best scores per beam size are bolded.}
  \label{tab:results}
\end{table}

\section{Related Work}
Neural probabilistic text generators are far from perfect; prior work has shown that they often generate text that is generic \cite{neural_conv, li-etal-2016-diversity}, unnatural \cite{holtzman2019curious}, and sometimes even non-existent \cite{stahlberg_nmt_2019}.
In the context of the degenerate behavior
of these models, the beam search curse---a specific phenomenon where using a larger beam size leads to worse performance---has been analyzed by a number of authors \cite{koehn-knowles-2017-six, murray-chiang-2018-correcting, yang-etal-2018-breaking, stahlberg_nmt_2019, jean-etal-2015-montreal, tu-etal-2016-modeling, he-2016-length, pmlr-v97-cohen19a}.
Many of these authors attribute the performance drop (as search becomes better) to an inherent bias in neural sequence models to pefer shorter sentences.  Other authors have ascribed fault to the model architectures, or how they are trained  \cite{cho-etal-2014-properties,bengio2015scheduled, sountsov-sarawagi-2016-length, vinyals_2017, ott2018analyzing, kumar2019calibration}.
To remedy the problem, a large number of regularized decoding objectives and modified training techniques have been proposed. In contrast, this work analyzes the behavior of neural text generators from a different angle: We provide a plausible answer---inspired by psycholinguistic theory---as to \emph{why} beam search (with small beams) leads to high-quality text, rather than another explanation of why exact search performs so badly.

\section{Conclusion}
We analyze beam search as a decoding strategy for text generation models by framing it as the solution to an exact decoding problem. We hypothesize that beam search has an inductive bias which can be linked to the promotion of uniform information density (UID), a theory from cognitive science regarding even distribution of information in linguistic signals. We observe a strong relationship between variance of surprisals (an operationalization of UID) and \bleu in our experiments with NMT models. With the aim of further exploring
decoding strategies for neural text generators in the context of UID, we design a set of objectives to explicitly encourage uniform information density in text generated from neural probabalistic models and find that they alleviate the quality degradation typically seen with increased beam widths.

\section*{Acknowledgments}
We would like to thank Ari Holtzman and Jason Eisner for useful feedback and discussion that helped improve this work.

\bibliography{anthology,acl2020}
\bibliographystyle{acl_natbib}
\newpage
\clearpage
\appendix
\onecolumn

\section{Theory}\label{sec:theory}
\begin{proof}
We prove \cref{sec:beam-limit} by induction. We denote the $\argmax$ of $\log \ptheta(\yy \mid \xx) - \lambda \cdot \calR_{\mathrm{greedy}}(\yy)$ as $\yy^\calR$ and the solution found by greedy search as $\yy^\mathrm{greedy}$. We will show that $y^\mathrm{greedy}_t = y^\calR_t$ for all $0 \leq t \leq \max(|\yy^\calR|, |\yy^\mathrm{greedy}|)$. 
The theorem holds trivially for the base case of $t=0$ because $y_0$ must be \bos for any valid hypothesis by definition of the hypothesis space (\cref{eq:candidates}).
Now, by the inductive hypothesis, suppose $y^\mathrm{greedy}_i = y^\calR_i$ for all $i < t$. We will show that our regularized objective must choose the same word as greedy search
at time-step $t$.
In the limiting case of \cref{eq:greedy}, the following function reflects the penalty to the distribution over tokens at position $t$:
\begin{equation}
    \lim_{\lambda \rightarrow \infty} \Big[
    \lambda \cdot \Big(u_t(y_t) - \min_{y' \in \vocab} u_t(y') \Big)^2 \Big] 
= \begin{cases} 0 & \textbf{if } u_t(y_t) = \min_{y' \in \vocab} u_t(y') \\
\infty & \textbf{otherwise} \end{cases} \nonumber
\end{equation} 
Since minimum surprisal implies maximum log-probability, the above function clearly returns either $0$ or $\infty$ depending on whether the decoding choice at time-step $t$ is greedy. Therefore the only choice that would not send the hypothesis score to $-\infty$ is the greedy choice, which implies any feasible solution to our objective must have $ y^\calR_t = y^\mathrm{greedy}_t$. 
By the principle of induction, $y^\mathrm{greedy}_t = y^\calR_t$ for all $0 \leq t \leq |\yy^\calR| = |\yy^\mathrm{greedy}|$, which in turn implies $\yy^\mathrm{greedy} = \yy^\calR$.
\end{proof}

\section{Parameters}\label{sec:parameters}
For values in \cref{fig:beam_bleu}, we perform grid search over $\lambda \in [0.2,0.5,0.7,1,2,3,4,6,7,8,9,10]$ and choose the $\lambda$ with the best validation set performance.
For combined UID regularization, we perform hyperparameter search over the 5 strength parameters, each sampled uniformly from the following values: $[0, 0.2,0.5,0.7,1,2,3,4,6,7,8,9,10]$. We run 50 trials on the validation set; $\lambda=5$ and $\lambda=2$ yield the best performance for the greedy and squared regularizers, respectively with all others $\lambda$ set to 0.

\begin{table}[ht!]
  \centering
  \adjustbox{max width=\linewidth}{
  \begin{tabular}{@{}lll@{}}
  \toprule
     & IWSLT'14 & WMT'14 \\
     
    \hline
     Greedy & 10 & 5\\
     Local Consistency &4 &6 \\
     Max &5 &3 \\
     Squared & 3& 2 \\
     Variance &7 &3\\
    
    \bottomrule
  \end{tabular} }
  \caption{$\lambda$ settings used during decoding in \cref{fig:beam_bleu} and reported in table \cref{tab:results}.}
  \label{tab:params}
\end{table}

\section{Additional Plots}
\begin{figure}[ht]
\floatbox[{\capbeside\thisfloatsetup{capbesideposition={left,top},capbesidewidth=7.25cm}}]{figure}[\FBwidth]
{\caption{\bleu vs. std. deviation of surprisals for translations generated with beam search on test sets of IWSLT'14 and WMT'14. Size of point indicates beam width used (between 5 and 100). In contrast to the subgraph of \cref{fig:bleu_surprisal}, the $x$-axis is not log-scaled.}\label{fig:beam_surp}}
{\includegraphics[width=8cm]{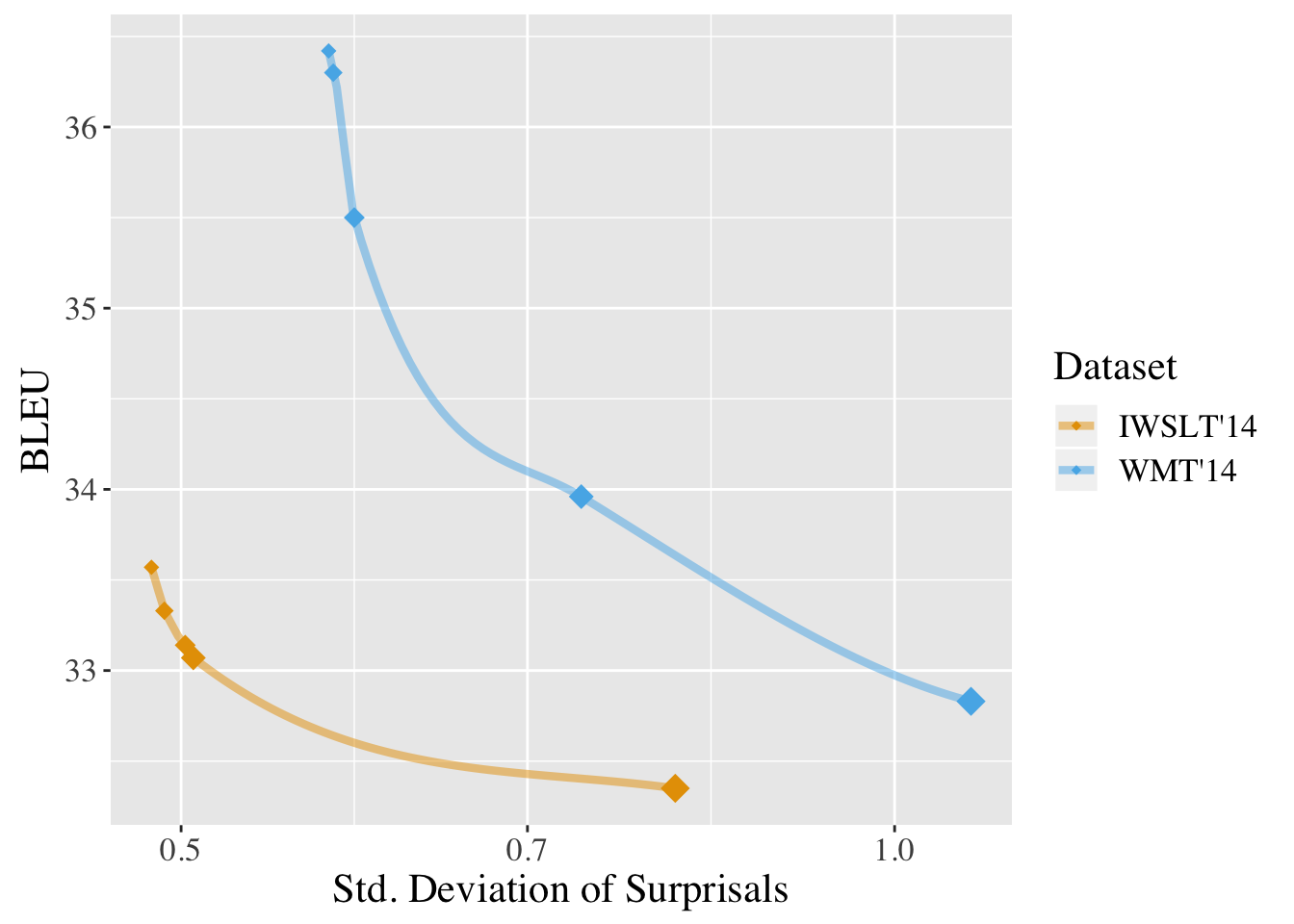}}
\end{figure}

\end{document}